\documentclass{article}

\usepackage{times}
\usepackage{graphicx} 
\usepackage{subfigure} 
\usepackage{caption}
\usepackage{enumitem}

\usepackage{natbib}
\renewcommand{\citealt}[1]{\citeauthor{#1},~\citeyear{#1}}
\usepackage{algorithm}
\usepackage{algorithmic}

\usepackage[draft]{hyperref}

\usepackage[accepted]{icml2017}

\usepackage{amsmath,amsthm,amssymb,amscd,latexsym}
\usepackage{mathtools}
\usepackage{relsize}
\usepackage{dsfont}

\newtheorem{proposition}{Proposition}
\newtheorem{definition}{Definition}

\newcommand{\n}{\mathbf{n}}
\newcommand{\y}{\mathbf{y}}
\newcommand{\x}{\mathbf{x}}

\newcommand{\mub}{\boldsymbol\mu}

\newcommand{\thetab}{\boldsymbol{\theta}}
\renewcommand{\Pr}{\ensuremath\mathrm{Pr}}

\newcommand{\M}{\mathcal{M}}
\newcommand{\X}{\mathcal{X}}
\renewcommand{\S}{\mathcal{S}}
\renewcommand{\L}{\mathcal{L}}
\newcommand{\T}{\mathcal{T}}
\newcommand{\I}{\mathbb{I}}
\newcommand{\E}{\mathbb{E}}
\newcommand{\R}{\mathbb{R}}
\newcommand{\Z}{\mathbb{Z}}
\DeclareMathOperator{\Var}{Var}
\DeclareMathOperator{\MSE}{MSE}
\DeclareMathOperator*{\argmax}{argmax}

\newcommand{\grad}{\nabla}
\renewcommand{\algorithmicrequire}{\textbf{Input:}}
\renewcommand{\algorithmicensure}{\textbf{Output:}}

\newcommand{\eat}[1]{}

\newcommand{\blue}[1]{\textcolor{blue}{#1}}

\newcommand{\dan}[1]{#1}
\newcommand{\savespace}[1]{}

\usepackage{tikz}
\usetikzlibrary{shapes,arrows,arrows.meta,positioning,decorations.pathmorphing,calc,fit, hobby}
\tikzstyle{rectangle}=[draw=black,thick,align=center,inner sep=5pt]
\tikzstyle{circ}=[circle,draw=black,thick,minimum size=3ex, inner sep=.4ex,align=center]
\tikzstyle{oval}=[ellipse,draw=black,thick,minimum size=3ex, inner sep=.9ex,align=center]
\tikzset{>={Latex[scale=2]}}
\pgfdeclarelayer{bg}    
\pgfsetlayers{bg,main}  

\icmltitlerunning{Differentially Private Learning of Graphical Models}

\begin{document} 

\twocolumn[
\icmltitle{Differentially Private Learning of \dan{Undirected} Graphical Models Using \dan{Collective Graphical Models}}

\icmlsetsymbol{equal}{*}

\begin{icmlauthorlist}
\icmlauthor{Garrett Bernstein}{umass}
\icmlauthor{Ryan McKenna}{umass}
\icmlauthor{Tao Sun}{umass}
\icmlauthor{Daniel Sheldon}{umass,mhc}
\icmlauthor{Michael Hay}{colgate}
\icmlauthor{Gerome Miklau}{umass}
\end{icmlauthorlist}
    
\icmlaffiliation{umass}{University of Massachusetts Amherst}
\icmlaffiliation{mhc}{Mount Holyoke College}
\icmlaffiliation{colgate}{Colgate University}
\icmlcorrespondingauthor{Garrett Bernstein}{gbernstein@cs.umass.edu}

\icmlkeywords{graphical models, differential privacy, collective graphical models, aggregate data}

\vskip 0.3in
]

\printAffiliationsAndNotice{}

\begin{abstract}
We investigate the problem of learning discrete, undirected graphical models in a differentially private way.
We show that the approach of releasing noisy sufficient statistics using the Laplace mechanism achieves a good trade-off between privacy, utility, and practicality. A naive learning algorithm that uses the noisy sufficient statistics ``as is'' outperforms general-purpose differentially private learning algorithms. However, it has three limitations: it ignores knowledge about the data generating process, rests on uncertain theoretical foundations, and exhibits certain pathologies. We develop a more principled approach that applies the formalism of collective graphical models to perform inference over the true sufficient statistics within an expectation-maximization framework. We show that this learns better models than competing approaches on both synthetic data and on real human mobility data used as a case study.
\end{abstract}

\section{Introduction} 
\label{sec:introduction}
Graphical models are a central tool in probabilistic modeling and machine learning. They pair expressive probability models with algorithms that leverage the graphical structure for efficient inference and learning. 
However, with data collection and modeling growing in importance in nearly all domains of society, there is increasing demand to apply graphical models in settings where the underlying data is sensitive and must be kept private. For example, consider applying graphical models to analyze electronic health records,
with the goal of guiding public health policy.
How can we derive these useful population-level outcomes without compromising the privacy of individuals?

\emph{Differential privacy} is a widely studied formalism for private data analysis~\cite{Dwork:2006aa}.
It provides a statistical privacy guarantee to individuals: the output of a differentially private algorithm is statistically nearly unchanged even if any single individual's record is added to or removed from the input data set.
The general idea is to carefully randomize the algorithm so that the (random) output does not depend too much on any individual's data.

Differentially private machine learning cleanly addresses the problem of extracting useful population-level models from data sets while protecting the privacy of individuals.
\dan{Indeed, this is an active and important research area (see Section~\ref{sec:related_work}), which includes private learning algorithms for a variety of general frameworks and specific machine learning models.}
This paper addresses the problem of privately learning parameters in a widely used class of probabilistic models: discrete, undirected graphical models.
\dan{Although our problem can be cast in terms of general private learning frameworks, these do not lead to practical algorithms. Previous work also addresses private learning for \emph{directed} graphical models (J.~\citealt{JZhang:2014aa}; Z.~\citealt{ZZhang:2016aa}).
Our problem of learning in undirected models, which are not locally normalized, is more general and substantially harder computationally.
}


To learn accurate models under differential privacy, it is critical to randomize the algorithm ``just enough'' to achieve the desired privacy guarantee without diminishing the quality of the learned model too much. 
This is usually done by modifying a learning algorithm to add noise to some intermediate quantity $X$, with the noise magnitude calibrated to the \emph{sensitivity} of $X$, a measure of how much $X$ can depend on any single individual's data in the worst case~\cite{Dwork:2006aa}. The randomization renders the noisy estimate of $X$ safe for release; all subsequent calculations using the noisy $X$, but not the original data, are also safe. 
Where should noise be injected into a machine learning algorithm to achieve the best utility? 
We highlight two high-level goals: \savespace{\begin{itemize}[leftmargin=*,itemsep=0pt]
\item Noise should be added at an ``information bottleneck'', so the sensitivity is as small as possible relative to the information being sought.\footnote{Sensitivity scales with the number of measurements: all else equal, a lower dimensional quantity will have lower sensitivity.}
\item Noise should be added to a quantity for which the sensitivity can be bounded tightly, so the noise magnitude can be kept as small as possible.
\end{itemize}} (1) Noise should be added at an ``information bottleneck'', so the sensitivity is as small as possible relative to the information being sought,\footnote{Sensitivity scales with the number of measurements; all else equal, a lower dimensional quantity will have lower sensitivity.} (2) noise should be added to a quantity for which the sensitivity can be bounded tightly, so the noise magnitude can be kept as small as possible.
These two principles are often at odds. For example, adding noise to the final learned parameters $\theta$ \citep[known as \emph{output perturbation};][]{Dwork:2006aa},
is appealing from the information bottleneck standpoint,
but if the learning algorithm is complex we may not be able to analyze the sensitivity and would have to rely on a coarse bound. Indeed, general private learning frameworks bound the sensitivity using quantities such as Lipschitz, strong-convexity, and smoothness constants~\cite{bassily2014private,Wu:2016aa}
or diameter of the parameter space~\cite{smith2008efficient}, which may be loose in practice.

In this paper we will take the approach of adding noise to the sufficient
statistics of a graphical model using the Laplace mechanism, \dan{a high-level approach that has also been applied recently for directed models~\cite{ZZhang:2016aa,Foulds:2016aa}}. This has a number of advantages. First, sufficient statistics, by definition, are an information bottleneck. Second, it is very easy to exactly analyze the sensitivity of sufficient statistics in graphical models, which are contingency tables. Third, adding Laplace noise to contingency tables prior to release is very simple, so it is reasonable to imagine adoption in practice, say, by public agencies.

However, it is not entirely clear how to learn parameters of a graphical model with \emph{noisy} sufficient statistics. One option, which we will refer to as \emph{naive MLE}, is to ignore the noise and conduct maximum-likelihood estimation as if we had true sufficient statistics. This works reasonably well in practice, and is competitive with or better than state-of-the-art general-purpose methods. In fact, we will show that naive MLE is consistent and achieves the same \emph{asymptotic} mean-squared error as non-private MLE. However, at reasonable sample sizes the error due to privacy is significant, and the approach has several pathologies~\citep[see also][]{yang2012differential,karwa2014differentially,karwa2016inference}, some of which make it difficult to apply in practice.
Therefore, we adopt a more principled approach of performing \emph{inference} about the true sufficient statistics within an expectation–maximization (EM) learning framework.

The remaining problem is how to conduct inference over sufficient statistics of a graphical model given noisy observations thereof. This is exactly the goal of inference in \emph{collective graphical models}~\citep[CGMs;][]{Sheldon:2011aa}, \dan{and we will adapt CGM inference techniques to solve this problem.}
Put together, our results significantly advance the state-of-the-art for privately learning discrete, undirected graphical models. We clarify the theory and practice of naive MLE. We show that it learns better models than existing state-of-the-art approaches in most scenarios across a broad range of synthetic tasks, and in experiments modeling human mobility from wifi access point data. We then show the more principled approach of conducting inference with CGMs is superior to competing approaches in nearly all scenarios.

\subsection{Related Work}
\label{sec:related_work}

\dan{
Differential privacy has been applied to many areas of machine learning, including learning specific models such as logistic regression~\cite{Chaudhuri:2009aa}, support vector machines~\cite{Rubinstein:2009aa}, and deep neural networks~\cite{Abadi:2016aa}; privacy in general frameworks such as empirical risk minimization~\citep[ERM; ][]{Chaudhuri:2011aa,kifer2012private,Jain:2013aa,bassily2014private}, gradient descent~\cite{Wu:2016aa}, and parameter estimation~\cite{smith2011privacy}; and theoretical analysis of what can be learned privately~\citep[e.g.,][]{blum2005practical,kasiviswanathan2011can}.

A key aspect of our work is conducting probabilistic inference over data or model parameters given knowledge of the probabilistic privacy mechanism and its output. \citet{karwa2014differentially,karwa2016inference} take a similar approach but for exponential random graph models, as do \citet{Williams:2010aa}, but for the factored exponential mechanism. Because sufficient statistics of graphical models are contingency tables, our work connects to the well-studied problem of releasing differentially private contingency tables~\cite{barak2007privacy,yang2012differential,hardt2012simple}; we adopt the Laplace mechanism because it is simple and fits well within our learning framework. 

We highlight connections between CGMs and differential privacy and adopt existing inference techniques for CGMs. In general, the inference problems we wish to solve are NP-hard~\cite{Sheldon:2013aa}, but a number of efficient approximate inference algorithms  are available~\cite{Liu:2014aa,Sun:2015aa,vilnis2015bethe}. In a paper that was primarily about CGM inference, \citet{Sun:2015aa} conducted a case study using CGMs to privately learn Markov chains; we build on this approach, which was limited in scope and did not address general graph structures.

Our work connects to an active current line of work on private probabilistic inference, some of which directly addresses learning in directed graphical models, but not the more challenging problem of learning in  undirected graphical models. Several closely related approaches, which we refer to as “One Posterior Sampling” (OPS), show that a single sample drawn from a posterior distribution is differentially private~\cite{Dimitrakakis:2014aa,Wang:2015aa,ZZhang:2016aa}. This can be understood as applying the exponential 
mechanism to the log-likelihood function, and can provide a point estimate for graphical model parameters~\cite{ZZhang:2016aa}. To apply OPS, one must sample from the posterior over parameters, $p(\Theta | X)$, which is straightforward for directed graphical models with conjugate priors, but not in undirected models, where posteriors over parameters are usually intractable. \citet{ZZhang:2016aa} and \citet{Foulds:2016aa} also developed fully Bayesian methods  using Laplace noise-corrupted sufficient statistics to update posterior parameters. Similar considerations apply to this approach, which matches ours in that it uses the same data release mechanism, but, like OPS, requires conjugate priors and thus easily applies only to directed graphical models. \citet{Wang:2015aa} also describe MCMC approaches to draw many private samples from a posterior distribution; this is another general framework that could apply to our problem, but, it relies on loose sensitivity bounds and since we only request point estimates, it would waste privacy budget by drawing many samples.
}



\section{Background and Problem Statement} 
\label{sec:problem_statement}

\newcommand{\C}{\mathcal{C}}
\def\alg{{\cal A}}
\def\db{\mathbf{X}}
\def\nbrs{\textrm{nbrs}}

We consider data sets consisting of $T$ discrete attributes associated with each individual. Let $x_t \in \X$ denote the value of the $t$th attribute of an individual; we assume for simplicity of notation that all variables take values in the same finite set $\X$.  Let $\x = (x_1, \ldots, x_T)$ denote the complete vector of attributes for an individual, and let $\mathbf{X} = (\x^{(1)}, \x^{(2)}, \ldots, \x^{(N)})$ denote a data set for an entire population of $N$ individuals.

\subsection{Differential Privacy}

Differential privacy offers strong privacy protection by imposing constraints on any algorithm that computes on the private dataset.  Informally, it requires that an individual's data has a bounded effect on the algorithm's behavior.  The formal definition requires reasoning about all pairs of datasets that are otherwise identical except one dataset contains one additional individual's data vector.
Let $\nbrs(\db)$ denote the set of datasets that differ from $\db$ by at most one individual's vector---i.e., if $\db' \in \nbrs(\db)$, 
then $\db' = (\x^{(1)}, \dots, \x^{({i-1})}, \x^{({i+1})}, \dots, \x^{(n)})$ for some $i$ or $\db' = (\x^{(1)}, \dots, \x^{({i})}, \x', \x^{({i+1})}, \dots, \x^{(n)})$ for some $i$ and some $\x' \in \X^T$.

\begin{definition}[Differential Privacy;~\citealt{Dwork:2006aa}] \label{def:dp}
A randomized algorithm $\alg$ satisfies $(\epsilon, \delta)$-differential privacy if for any input $\db$, any $\db' \in \nbrs(\db)$ and any subset of outputs $S \subseteq \textrm{Range}(\alg)$, 
$$ \Pr[\alg(\db) \in S] \leq \exp(\epsilon) \Pr[\alg(\db') \in S] + \delta.$$
\end{definition}
When $\delta = 0$, we say that the algorithm satisfies $\epsilon$-differential privacy.  All of the algorithms we propose satisfy $\epsilon$-differential privacy but we compare against some algorithms that satisfy the weaker condition of $(\epsilon, \delta)$-differential privacy with non-zero $\delta$.

We achieve differential privacy by injecting noise into the statistics that are computed on the data.   Let $f$ be any function that maps datasets to $\R^d$.  The amount of noise depends on the {\em sensitivity} of $f$.  

\begin{definition}[Sensitivity] \label{def:sensitivity}
The {\em sensitivity} of a function $f$ is defined as 
$\Delta_{f} = \max_{\db} LS_{f}(\db)$
where $LS_{f}$ denotes the {\em local sensitivity} of $f$ on input $\db$ and is defined as 
$LS_{f}(\db) = \max_{\db' \in \nbrs(\db)} \| f(\db) - f(\db') \|_1.$
\end{definition}
We drop the subscript $f$ when it is clear from context.

Our approach achieves differential privacy through the application of the Laplace mechanism.
\begin{definition}[Laplace Mechanism;~\citealt{Dwork:2006aa}] \label{def:laplace}
Given function $f$ that maps datasets to $\R^d$, the Laplace mechanism is defined as $\mathcal{L}(\db) = f(\db) + \mathbf{z}$ where $\mathbf{z} = (z_1, \dots, z_d)$ and each $z_i$ is an i.i.d. random variable from $\text{Laplace}(\Delta_f/\epsilon)$.
\end{definition}

An important property of differential privacy is that any additional post-processing on the output cannot weaken the privacy guarantee.
\begin{proposition}[Post-processing;~\citealt{Dwork14Algorithmic}]
Let $\alg$ be an $(\epsilon, \delta)$-differentially 
private algorithm that maps datasets to $\R^d$ and 
let $g: \R^d \rightarrow \R^{d'}$ 
be an arbitrary function.  Then $g \circ \alg$ is also $(\epsilon, \delta)$-differentially private.
\end{proposition}

\subsection{Problem Statement}

Our goal is to learn a probabilistic model $p(\x)$ from the data set $\mathbf{X}$ while protecting the privacy of individuals.
We will learn probability distributions $p(\x)$ that are \emph{undirected
discrete graphical models} \citep[also called Markov random fields;][]{koller2009probabilistic}.
These are defined by a set of local \emph{potential functions} of the form
$\psi_C(\x_C)$, where $C \subseteq \{1, \ldots, T\}$ is an index set or
\emph{clique}, $\x_C$ is a subvector of $\x$ corresponding to $C$, and $\psi_C:
\X^{|C|} \rightarrow \R^+$ assigns a \emph{potential} value to each possible
$\x_C$. The probability model is $p(\x) = \frac{1}{Z} \prod_{C \in \C} \psi_C(\x_C)$
where $\C$ is the collection of cliques that appear in the model, and $Z =
\sum_{\x}\prod_{C \in \C} \psi_C(\x_C)$ is the normalizing constant or
\emph{partition function}. The graph $G$ with node set $V = \{1, \ldots, T\}$
and edges between any two indices that co-occur in some $C \in \C$ is the
\emph{independence graph} of the model; therefore, each index
set $C$ is a clique in $G$.

For learning, it is most convenient to express the model in log-linear or
exponential family form as:
\begin{equation}
\label{eq:log-linear-model}
p(\x ; \thetab) = \exp\Bigg\{\sum_{C \in \C} \sum_{i_C \in \X^{|C|}} \I\{\x_C = i_C\} \theta_{C}(i_C) - A(\thetab)\Bigg\}.
\end{equation}
In this expression: $\I\{\cdot\}$ is an indicator function; the variable $i_C
\in \X^{|C|}$ denotes a particular setting of the variables $\x_C$; the
\emph{parameters} $\theta_{C}(i_C) = \log \psi_{C}(i_C)$ are log-potential
values; the vector $\thetab \in \R^d$ is the concatenation of all parameters;
and $A(\thetab) = \log Z(\thetab)$ is the log-partition function, with the
dependence of $Z$ on the parameters now made explicit. Note that, for any
$\thetab \in \R^d$, the density is strictly positive: $p(\x; \thetab) > 0$ for
all $\x$. \dan{This is true because the potential values $\psi_C(i_C)$ are strictly
positive, so the log-potentials are finite.}

The goal is to learn parameters $\hat{\thetab}$ from the data $\mathbf{X}$ in a
way that is $\epsilon$-differentially private and such that $p(\x;
\hat{\thetab})$ is as accurate as possible. We will measure accuracy as
Kullback-Leibler divergence from an appropriate reference
distribution~\cite{kullback1951information}.
In synthetic experiments, we will measure the divergence $D\big( p(\cdot;
\thetab) \| p(\cdot; \hat{\thetab})\big)$, where $p(\x; \thetab)$ is the true
density. For real data, we will measure the holdout log-likelihood
$E_q\big[\log p(\x; \hat{\thetab})\big]$ where $q$ is the empirical
distribution of the holdout data, which is equal to a constant minus $D\big( q
\| p(\cdot; \hat{\thetab})\big)$.

The problem of privately selecting which cliques to include in the model (i.e.,
\emph{model selection} or \emph{structure learning}) is interesting but not
considered in this paper; we assume the cliques $\C$ are fixed in advance by
the modeler.

\section{Approach} 
\label{sec:approach}

To develop our approach to privately learn graphical model parameters, we
first discuss standard concepts related to maximum-likelihood estimation
for graphical models.

\vspace{2pt}
\textbf{Log-Likelihood, Sufficient Statistics, Marginals.} From
Eq.~\eqref{eq:log-linear-model}, the log-likelihood $\mathcal{L}(\thetab) =
\log \prod_{i=1}^N p\big(\x^{(i)}; \thetab\big)$ of the entire data set can be
written as
\begin{align*}
\mathcal{L}(\thetab)
&= \Bigg[\sum_{C \in \C} \sum_{i_C \in \X^{|C|}} n_C(i_C) \theta_{C}(i_C)\Bigg]
- N  A(\thetab)
\end{align*}
where $n_C(i_C) = \sum_{i=1}^N \I\{\x^{(i)}_C = i_C\}$ is a count of how many
times the configuration $i_C$ for the variables in clique $C$ appears in the
population. The collection of counts $\n_C = \big( n_C(i_C) \big)$ for all
possible $i_C$ is the (population) \emph{contingency table} on clique $C$. Let
$\n$ denote the vector concatenation of the contingency tables for all cliques.
Then we can rewrite the log-likelihood more compactly as
\begin{equation}
\label{eq:log-likelihood}
\mathcal{L}(\thetab) = f(\n, \thetab) := \thetab^T \n - N A(\thetab)
\end{equation}

The \dan{most common} approach for parameter learning in graphical models is maximum
likelihood estimation: find the parameters $\hat{\thetab}$ that maximize
$\mathcal{L}(\thetab)$. The resulting parameter vector $\hat{\thetab}$ is a
\emph{maximum-likelihood estimator} (MLE). It is clear from
Eq.~\eqref{eq:log-likelihood} that this problem depends on the data only
through the contingency tables $\n$.
Indeed, the clique contingency tables $\n$ are \emph{sufficient statistics} of
the model: they measure all of the information from the data set $\mathbf{X}$
that is relevant for estimating the parameter $\thetab$~\cite{Fisher:1922aa}.

The algorithmic approach for maximum-likelihood estimation in graphical models
is standard~\cite{koller2009probabilistic}, and we do not repeat the details here. However, there are
a few concepts that are important for our development.
The \emph{marginals} of a graphical model are the marginal probabilities
$\mu_C(i_C) = p(\x_C = i_C; \thetab)$ for all cliques $C$ and configurations
$i_C$. Let $\mub$ be the vector concatenation of all marginals, and note that
$\mub = \E_{\thetab}[\n]/N$. Similarly, let $\hat{\mub} = \n / N$ be the
\emph{data marginals}---these are marginal probabilities of the empirical
distribution of the data.

Marginals play a fundamental role in estimation. First, note that we can divide
Eq.~\eqref{eq:log-likelihood} by $N$ to see that the MLE only depends on the
data through the data marginals $\hat{\mub}$. However, we leave $\L(\thetab)$
in the current form because it is more convenient for the CGM development in
Section~\ref{sub:collective_graphical_models}. Second, it is well known that
$
\grad_{\thetab} \L(\thetab) = N(\hat{\mub} - \mub),
$
so maximum likelihood estimation seeks to adjust $\thetab$ so that the data and
model marginals match. Third, it can (almost) always succeed in doing so,
even if the data marginals do not come from a graphical model. More formally,
let $\M$ be the \emph{marginal polytope}: the set of all vectors $\mub$ such
that there exists some distribution $q(\x)$ with marginal probabilities $\mub$.

\begin{proposition}[\citealt{wainwright2008graphical}]
\label{prop:marginals}
For any $\mub$ in the interior of $\M$, there is a unique distribution $p(\x;
\thetab)$ with marginals $\mub$, i.e., such that $\mub = E_{\thetab}[\n]/N$.
\end{proposition}
\dan{Applying Proposition~\ref{prop:marginals} to the data marginals
$\hat{\mub}$ shows that} if these belong to the interior of $\M$, we may learn a
distribution with marginals that match what we observe in the data. Note that,
while the \emph{distribution} $p(\x; \thetab)$ is unique, the parameters
$\thetab$ are not, because our model is overcomplete.
If $\mub$ belongs to $\M$ but not the interior of $\M$, which occurs, for
example, when some marginals are zero, the situation is more complex: there is
no (finite) $\thetab \in \R^d$ such that $p(\x; \thetab)$ has marginals
$\mub$.\footnote{However, there is a sequence $\{\thetab^k\}$ where $\thetab^k
\in \R^d$ and $\displaystyle\lim_{k \rightarrow \infty} \E_{\thetab^k}[\n]/N = \mub$.}
Similarly, the MLE does not exist, meaning that its maximum is not attained
for any finite $\thetab$~\cite{fienberg2012maximum,Haberman1973}. This issue will end up being significant in our understanding of the naive MLE approach in the following section.


%
%

\subsection{Noisy sufficient statistics} 

From the development so far, there are two obvious possibilities for
randomizing the learning process to achieve privacy: \savespace{(1) (Output perturbation) Find the MLE $\hat{\thetab}$ and add Laplace noise proportional to its sensitivity, (2) (Sufficient statistics perturbation) Add Laplace noise to the sufficient statistics $\n$, and then conduct maximum-likelihood estimation.}
\begin{enumerate}[leftmargin=*,itemsep=0pt]
\item (Output perturbation) Find the MLE $\hat{\thetab}$ and add Laplace noise
  proportional to its sensitivity.
\item (Sufficient statistics perturbation) Add Laplace noise to the sufficient
  statistics $\n$, and then conduct maximum-likelihood estimation.
\end{enumerate}

The two approaches are similar from an information bottleneck standpoint---the
dimensionality of $\n$ and $\hat{\thetab}$ is the same. However, the
sensitivity of $\hat{\thetab}$ is difficult to analyze, since it requires reasoning about worst-case inputs. It also may be high due to pathological inputs whose local sensitivity is much higher than that of realistic data sets.
On the other hand, the
sensitivity of $\n$ is very easy to analyze and the analysis is tight: the local sensitivity is the same for all data sets.

\begin{proposition}
\label{prop:sensitivity}
Let $\n(\mathbf{X})$ be the sufficient statistics of a graphical model with clique set $\C$ on data set $\mathbf{X}$. The local sensitivity of $\n$ is $|\C|$ for all inputs $\mathbf{X}$. Therefore the sensitivity of $\n$ is $|\C|$.
\end{proposition}
(Proofs can be found in the supplementary material.)
So, a simple approach to achieve privacy is to release noisy sufficient
statistics $\y$ that are obtained after applying the Laplace mechanism:
\begin{equation}
\label{eq:laplace}
y_C(i_C) = n_C(i_C) + \text{Laplace}\big(|\C|/\epsilon\big)
\end{equation}




\paragraph{Positive results.} How can we learn with noisy sufficient statistics
$\y$? A naive approach is to use $\y$ in place of $\n$ in maximum-likelihood
estimation, i.e., to find $\hat{\thetab}$ to maximize $f(\y, \thetab)$. The
validity of this approach has been debated in the
literature~\cite{yang2012differential}. However, it is relatively easy to show
that it behaves well asymptotically.

\begin{proposition}
\label{prop:mse}
Assume $\x^{(1)}, \ldots, \x^{(N)}$ are drawn iid from a probability
distribution with marginals $\mub$.
\dan{The marginal estimate $\bar{\mu}_C(i_C) = \frac{1}{N}y_C(i_C)$ obtained
from the noisy sufficient statistics is unbiased and consistent, with mean squared error:}
\begin{equation}
\label{eq:mse}
\MSE\big(\bar{\mu}_C(i_C)\big) =
\frac{\mu_C(i_C)\big(1 - \mu_C(i_C)\big)}{N}  + \frac{2|\C|^2}{N^2 \epsilon^2}
\end{equation}
Now let $\hat{\thetab} \in \argmax_{\thetab}f(\y, \thetab)$ be parameters estimated
using the noisy sufficient statistics $\y$.
If the true distribution $p(\x; \thetab)$ is a graphical model with cliques
$\C$, then the estimated distribution $p(\x; \hat{\thetab})$ converges to
$p(\x; \thetab)$.
\end{proposition}

\eat{ \blue{Consider expanding footnote below into ``Remark'' here about why
  these guarantees do not carry over to \emph{learned} marginals --- because
estimation theory requires that marginals belong to the marginal polytope.} }

\begin{figure*}
\centering
\subfigure[]{
\includegraphics[height=1.3in]{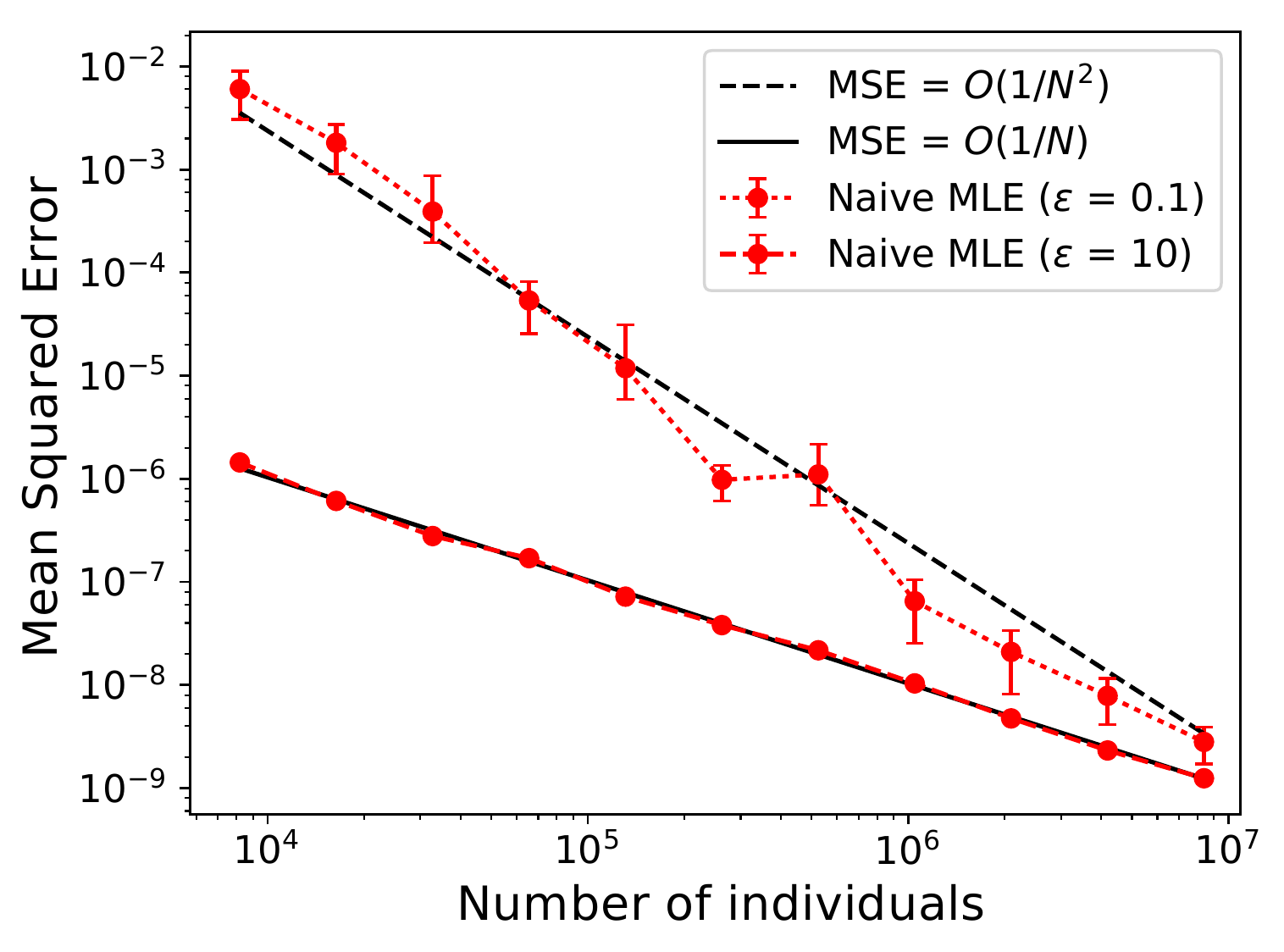}
\label{subfig:noisy-suff-stat-a}
}
\subfigure[]{
\includegraphics[height=1.3in]{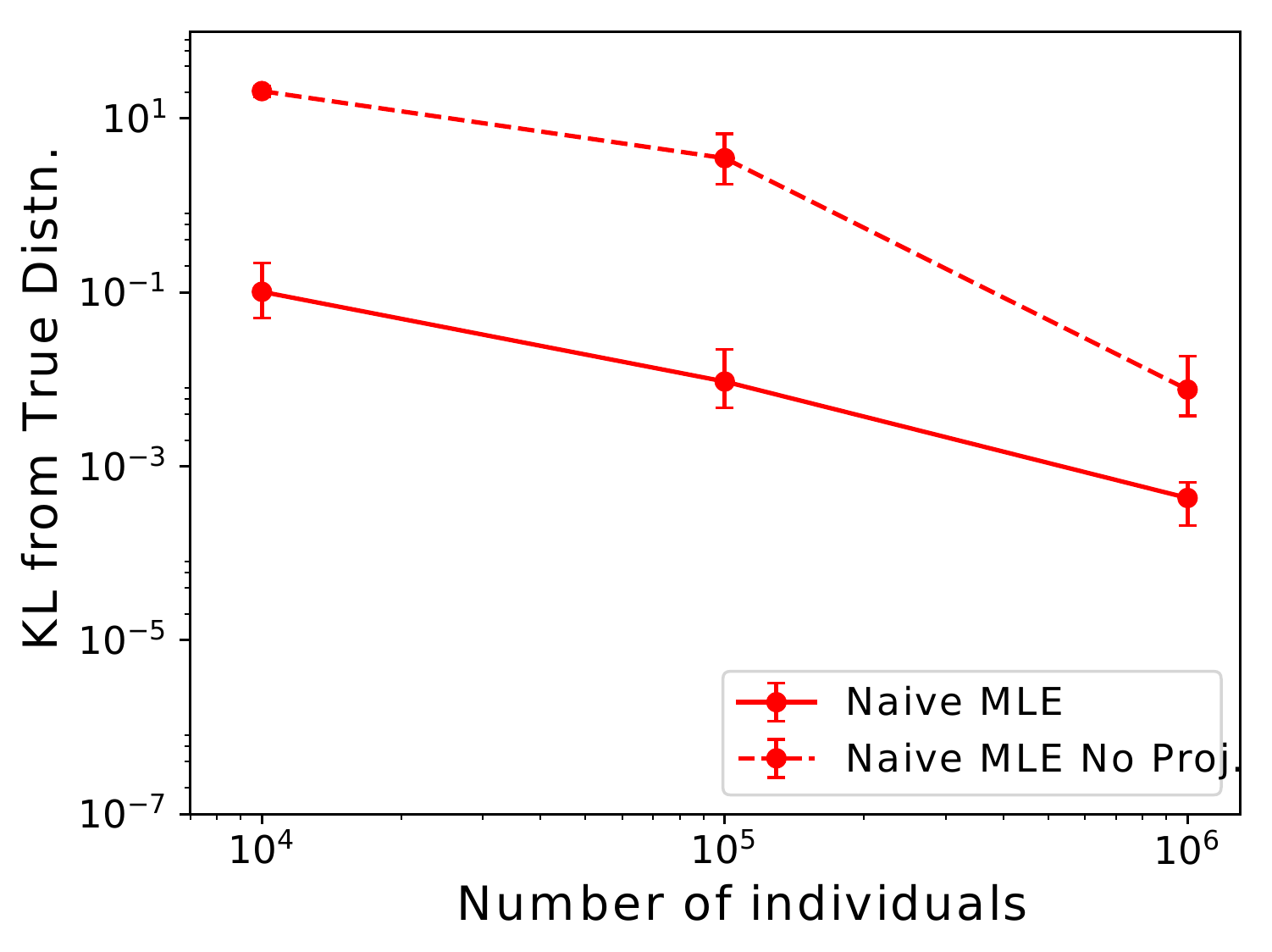}
\label{subfig:noisy-suff-stat-b}
}
\subfigure[]{
\includegraphics[height=1.3in]{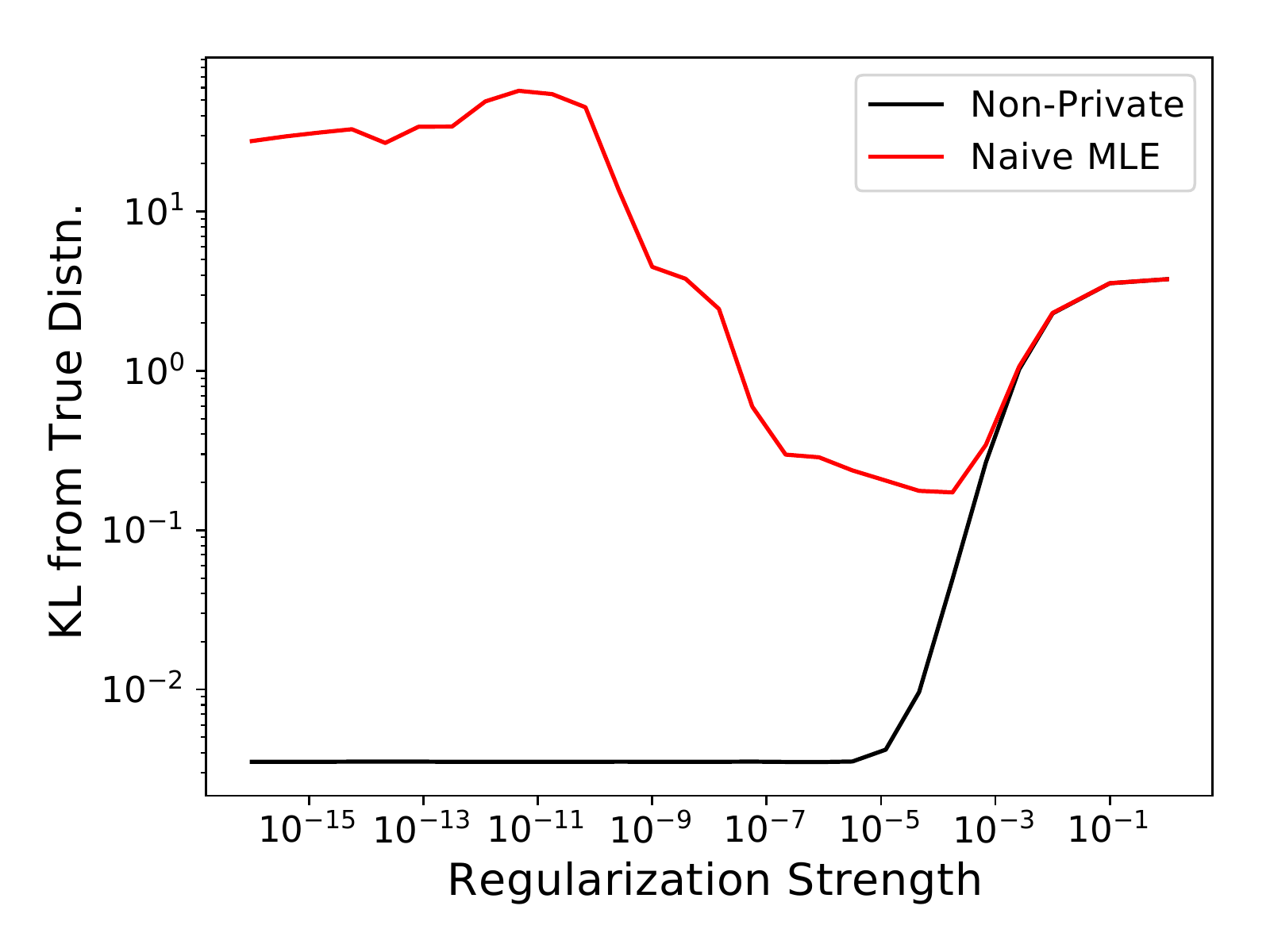}
\label{subfig:noisy-suff-stat-c}
}
\caption{Sample results on synthetic data illustrating behavior of naive MLE (see Section \ref{sub:synthetic_data} for experiment details): (a) MSE of learned marginals vs population size $N$ on a chain model with $T = 10$, $|\X|=10$; reference lines indicate predicted slope for $O(1/N)$ and $O(1/N^2)$ error terms, respectively (the function $c/N^d$ has slope $-d$ on a log-log plot), (b) effect of projecting marginals on performance of naive MLE for an Erd\H{o}s-R\'{e}yni graph with $T=10$, $|\X|=20$, $\epsilon = 0.5$, (c) effect of regularization on KL-divergence for learning with and without privacy; chain model with $T=10$, $|\X|=10$, $\epsilon = 0.1$. \label{fig:noisy-suff-stat}}
\end{figure*}

\paragraph{Pathologies.}
Asymptotically, the noisy sufficient statistics behave as desired in terms of
MSE: the $O(1/N)$ term, which is due to sampling error and not privacy,
dominates for large $N$. However, for practical settings of $\epsilon$ the
$O(1/N^2)$ term, which is due to privacy, is dominant until $N$ becomes very large,
due to the large constant $2|\C|^2/\epsilon^2$.
Figure~\ref{subfig:noisy-suff-stat-a} illustrates this issue. \savespace{For large $\epsilon$, the $O(1/N)$ sampling error is dominant; however, for smaller $\epsilon$, the $O(1/N^2)$ privacy error term is dominant even for $N$ approaching $10^7$.} \savespace{\footnote{Note
  that Proposition~\ref{prop:marginals} suggests that the MSE results for the
  \emph{estimated} marginals $\hat{\mub}$ will carry over to marginals of the
  \emph{learned} model $p(\x; \hat{\thetab})$. However the situation is
  complicated by the fact that $\hat{\mub}$ does not belong to the marginal
  polytope. Despite this, we observe in practice that the MSE of the learned
  marginals follow the predictions of Proposition~\ref{prop:mse}.}}

A second pathology is that the noise added for privacy destroys some of the
structure expected in the empirical marginals. The true data marginals
$\hat{\mub} = \n/N$ belong to the marginal polytope: in particular, this means
that each clique marginal $\hat{\mub}_C$ is nonnegative and sums to one, and
that clique marginals agree on common subsets of variables. After adding noise,
the \emph{pseudo}-marginals $\bar{\mub} = \y/N$ do not belong to the marginal
polytope: $\bar{\mub}$ may have negative values, and does not satisfy
consistency constraints. We find that a partial fix is very helpful
empirically: project the pseudo-marginal $\bar{\mub}_C$ for each clique onto
the simplex prior to conducting MLE, which can be done via a standard
procedure~\cite{Duchi:2008aa}. Let $\tilde{\mub}$ be the projected marginals. We
now have that $\tilde{\mub}_C$ is a valid marginal for each clique $C$, but
consistency constraints are not satisfied among cliques, and it is still the
case that $\tilde{\mub} \notin \M$. Figure~\ref{subfig:noisy-suff-stat-b}
illustrates the benefits of projection on the quality of the model learned by Naive MLE.

A more significant pathology has to do with zeros in the projected marginals
$\tilde{\mub}$, which are more prevalent than in true data marginals
$\hat{\mub}$. This is because the addition of Laplace noise creates negative
values, which are then truncated to zero during projection.
As discussed following Proposition~\ref{prop:marginals}, zero values in the
marginals lead to non-existence of the
MLE~\cite{fienberg2012maximum,Haberman1973}.
If $\tilde{\mu}_C(i_C) = 0$, the likelihood increases monotonically as $\theta_C(i_C)$ goes to negative
infinity; in other words, the model attempts to drive the learned marginal
probability to zero.
Numerically, we can address this by regularization, e.g., adding  $\lambda
\|\thetab\|^2$ to the objective function for arbitrarily small $\lambda > 0$.
However, we may still learn vanishingly small marginal probabilities, which can
lead to a very large KL-divergence between the true and learned models.
Figure~\ref{subfig:noisy-suff-stat-c} illustrates the effect of $\lambda$ on
KL-divergence with both noisy sufficient statistics and true sufficient
statistics. At high $\lambda$ (strong regularization), both methods underfit and yield poor KL divergence. Learning with true sufficient statistics has no tendency to overfit; it achieves good performance for a broad range of $\lambda$ approaching zero. Naive MLE with noisy sufficient statistics overfits badly (to zeros) for small $\lambda$, and must be tuned ``just right'' to achieve reasonable performance.



\subsection{Collective Graphical Models} 
\label{sub:collective_graphical_models}

Since learning with noisy sufficient statistics ``as-is'' 
has several pathologies and
is less robust than maximum-likelihood estimation in the absence of privacy, we investigate 
a more principled approach, which matches the data generating process:  We
treat the true sufficient statistics $\n$ as latent variables, and learn
$\thetab$ to maximize the \emph{marginal} likelihood $p(\y ; \thetab) =
\sum_{\n} p(\n, \y; \thetab)$. In this section, we will develop an EM approach
to accomplish this.

\renewcommand{\mid}{\,|\,}
In EM, we need to conduct inference to compute $\E[\n \mid \y; \thetab]$ for a
fixed value of $\thetab$.
\dan{This is the central problem of \emph{collective graphical models} (CGMs)~\cite{Sheldon:2011aa}.
Consider the joint distribution $p(\n, \y; \thetab) = p(\n; \thetab) p(\y \mid \n)$,
which we use to compute $\E[\n \mid \y; \thetab]$.
The noise mechanism $p(\y \mid \n)$ arises directly from the Laplace mechanism (see Eq.~\eqref{eq:laplace}).
The distribution of the sufficient statistics, $p(\n; \thetab)$, is
known as the \emph{CGM distribution}. It can be written in closed form when
the model is \emph{decomposable}, i.e., the cliques $\C$ correspond to the nodes of some junction tree $\T$.}
Although decomposability is a significant restriction, let us assume that such a tree $\T$ exists;
we will use the exact results derived for this case to develop an approximation
for the general case. Let $\S$ be the set of separators of $\T$, and let
$\nu(S)$ be the multiplicity of $S \in \S$, i.e., the number of distinct edges
$(C_i, C_j) \in \T$ for which $S = C_i \cap C_j$. Under these assumptions, the
CGM distribution has the form~\cite{Liu:2014aa}:
\begin{align*}
p(\n; \thetab) &= h(\n) \cdot \exp\big(f(\n, \thetab) \big), \\
h(\n) &= N! \cdot
      \frac{\displaystyle \prod_{S \in \S} \prod_{i_S \in \X^{|S|}} \big(n_S(i_S)!\big)^{\nu(S)}}
           {\displaystyle \prod_{C \in \C} \prod_{i_C \in \X^{|C|}} n_C(i_C)!} \cdot \I\{\n \in \M^{\mathbb{Z}}_N\}
\end{align*}
The term $\exp\big(f(\n, \thetab)\big)$ is the probability of an ordered data
set $\mathbf{X}$ with sufficient statistics $\n$, as discussed previously. The
term $h(\n)$ is a base measure that counts the number of ordered data sets with
sufficient statistics equal to $\n$, and enforces constraints on $\n$. The
\emph{integer-valued marginal polytope} $\M^{\Z}_N$ is the set of all vectors
$\n$ that are sufficient statistics of some data set $\mathbf{X}$ of size $N$.

Exact inference in CGMs is intractable~\cite{Sheldon:2013aa}. Therefore, it is
typical to relax the integrality constraint and apply Stirling's approximation:
$\log n! \approx n \log n - n$. Let $\M_N$ be the feasible set with the
integrality constraint removed, which is now just the standard marginal
polytope scaled so that each marginal sums to $N$ instead of one.

\begin{proposition}[\citealt{Sun:2015aa,nguyen2016approximate}]
\label{prop:cgm-approx}
For a decomposable CGM with junction tree $\T$, the following approximation of the CGM
log-density for any $\n \in \M_N$ is obtained by applying Stirling's
approximation:
\begin{equation}
\label{eq:cgm-approx}
\log p(\n, \y; \thetab) \approx \thetab^T \n - N A(\thetab) + H(\n) + \log p(\y| \n).
\end{equation}
Here, $H(\n) = -N\sum_{\x} q(\x) \log q(\x)$ is the entropy of the unique
distribution $q(\x) = p(\x; \thetab)$ in the graphical model family with marginals
equal to $\n/N$.
\end{proposition}

Proposition~\ref{prop:cgm-approx} is the basis for \emph{approximate MAP
inference problem in CGMs}: find $\n$ to maximize Eq.~\eqref{eq:cgm-approx} and
obtain an approximate mode of $p(\n \mid \y; \thetab)$.
\dan{Even though our goal is to compute the \emph{mean} $\E[\n \mid \y; \thetab]$, it has been
shown that the approximate mode, which is also a real-valued vector, is an excellent
approximation to the mean for use within the EM algorithm~\cite{Sheldon:2013aa}.
Note that for non-decomposable models, we will simply apply the same approximation as in
Proposition~\ref{prop:cgm-approx}, even though an exact expression for the
counting measure $h(\n)$, and therefore the correspondence of $\log h(\n)$ to
an entropy $H(\n)$, is not known in this case.
}
Then, after dropping the term $N A(\thetab)$ from
Proposition~\ref{prop:cgm-approx}, which is constant with respect to $\n$, the
approximate MAP problem can be rewritten as:
\begin{equation}
\label{eq:approximate-map}
\n^* \in \argmax_{\n \in \M_N}\ \thetab^T \n + H(\n) + \log p(\y \mid \n)
\end{equation}
This equation reveals a close connection to variational principles for
graphical models~\cite{wainwright2008graphical}. It is identical to the
variational optimization problem for marginal inference in standard graphical
models, except the objective has an additional term $\log p(\y | \n)$, which
is non-linear in $\n$. Several message-passing based algorithms have been
developed to efficiently solve the approximate MAP problem. For trees or
junction trees, Problem~\eqref{eq:approximate-map} is convex as long as $\log
p(\y |\n)$ is concave in $\n$ (which is true in most cases of interest, such as Laplace noise) so it can be solved
exactly~\cite{Sun:2015aa,vilnis2015bethe}. For loopy models, both the entropy
$H(\n)$ and the feasible set $\M_N$ must be
approximated~\cite{nguyen2016approximate}.

Algorithm~\ref{alg:nlbp} shows pseudocode \emph{non-linear belief
propagation} \citep[NLBP;][]{Sun:2015aa}, which we select as our primary inference
approach due to its simplicity.
It is a thin wrapper around standard BP, and can be applied to trees, in which
case it exactly solves Problem~\eqref{eq:approximate-map}, or it can be applied
to loopy graphs by using loopy BP (LBP) as the subroutine, in which case it is
approximate.

Our final EM learning procedure is shown in Algorithm~\ref{alg:em}. It
alternates between inference steps that solve the approximate MAP problem to
find $\n_t \approx \E[\n \mid \y; \theta_t]$, and optimization steps to
re-estimate parameters given the inferred sufficient statistics $\n_t$. See
also~\cite{Sheldon:2013aa,Liu:2014aa,Sun:2015aa}.

\newcommand{\nlbp}[1]{
\begin{algorithm}[#1]
\caption{Non-Linear Belief Propagation (NLBP)}
\label{alg:nlbp}
\begin{algorithmic}
\REQUIRE $\thetab$, $\y$, damping parameter $\alpha > 0$
\WHILE{$\neg$ converged}
\STATE $\thetab' \leftarrow \thetab + \grad_\n \log p(\y \mid \n)$
  \vskip 2pt
\STATE $\n' \leftarrow \text{STANDARD-BP}\big( \thetab' \big)$ \COMMENT{\small{Normalized to sum to $N$}}
  \vskip 2pt
\STATE  $\n \leftarrow (1-\alpha)\n + \alpha \n'$
\ENDWHILE
\end{algorithmic}
\end{algorithm}
}

\newcommand{\emalg}[1]{
\begin{algorithm}[#1]
\caption{EM for CGMs}
\label{alg:em}
\begin{algorithmic}
\REQUIRE Noisy sufficient statistics $\y$
\STATE Initialize $\thetab_0$ arbitrarily
\WHILE{$\neg$ converged}
\STATE $\n_t \leftarrow \text{NLBP}(\thetab_t, \y)$
  \vskip 2pt
\STATE $\thetab_{t+1} \leftarrow \argmax_{\thetab} \thetab^T \n_t - N A(\thetab)$
  \vskip 2pt
\ENDWHILE
\end{algorithmic}
\end{algorithm}
}

\nlbp{t}
\emalg{t}




\section{Experiments} 
\label{sec:experiments}

We conduct a number of experiments on synthetic and real data to evaluate the quality of models learned by both Naive~MLE and CGM.

\textbf{Methods.} 
\label{sub:private_stochastic_gradient_descent}
We compare three algorithms: Naive MLE, CGM, and a version of private stochastic gradient descent (PSGD) due to \citet{Abadi:2016aa}.
PSGD belongs to a class of general-purpose private learning algorithms that can be adapted to our problem, including gradient descent or stochastic gradient descent algorithms for empirical risk minimization~\cite{Chaudhuri:2011aa,kifer2012private,Jain:2013aa,bassily2014private,Abadi:2016aa} and the subsample-and-aggregate approach for parameter estimation~\cite{smith2011privacy}. We chose PSGD because it is a state-of-the-art method and it significantly outperformed other approaches in preliminary experiments.
However, note that PSGD satisfies only \emph{$(\epsilon, \delta)$-differential privacy} for $\delta>0$, which is a weaker privacy guarantee than $\epsilon$-differential privacy. We tune PSGD using a grid search over all relevant parameters to ensure it performs as well as possible.




\subsection{Synthetic data} 
\label{sub:synthetic_data}

\begin{figure*}[t]
        \centering


        \subfigure[$\epsilon = 0.01$]{
            \includegraphics[width=.23\textwidth]{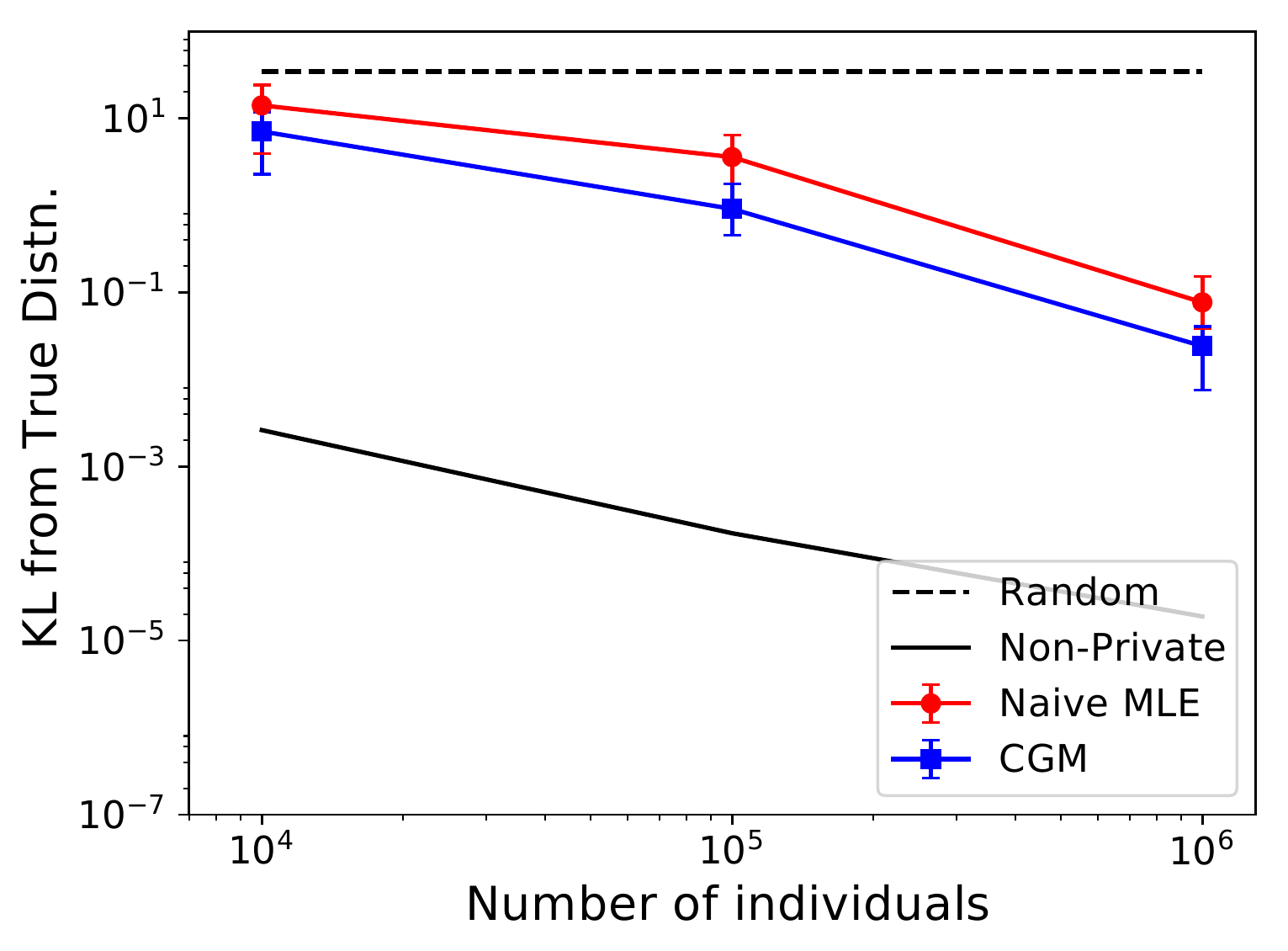}
            \label{subfig:thirdchain_3}
        }
        \subfigure[$\epsilon = 0.1$]{
            \includegraphics[width=.23\textwidth]{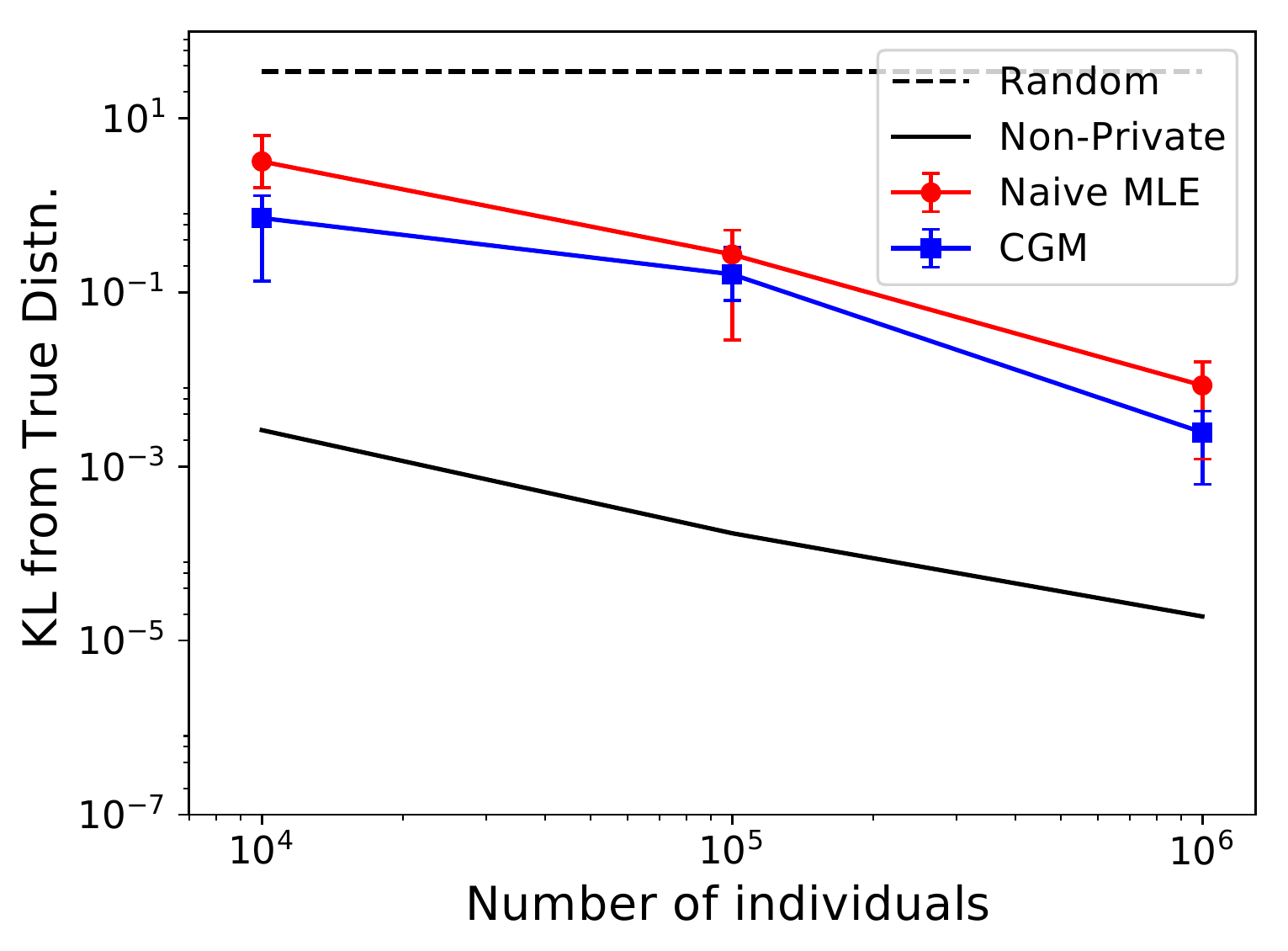}
            \label{subfig:thirdchain_2}
        }
        \subfigure[$\epsilon = 0.5$]{
            \includegraphics[width=.23\textwidth]{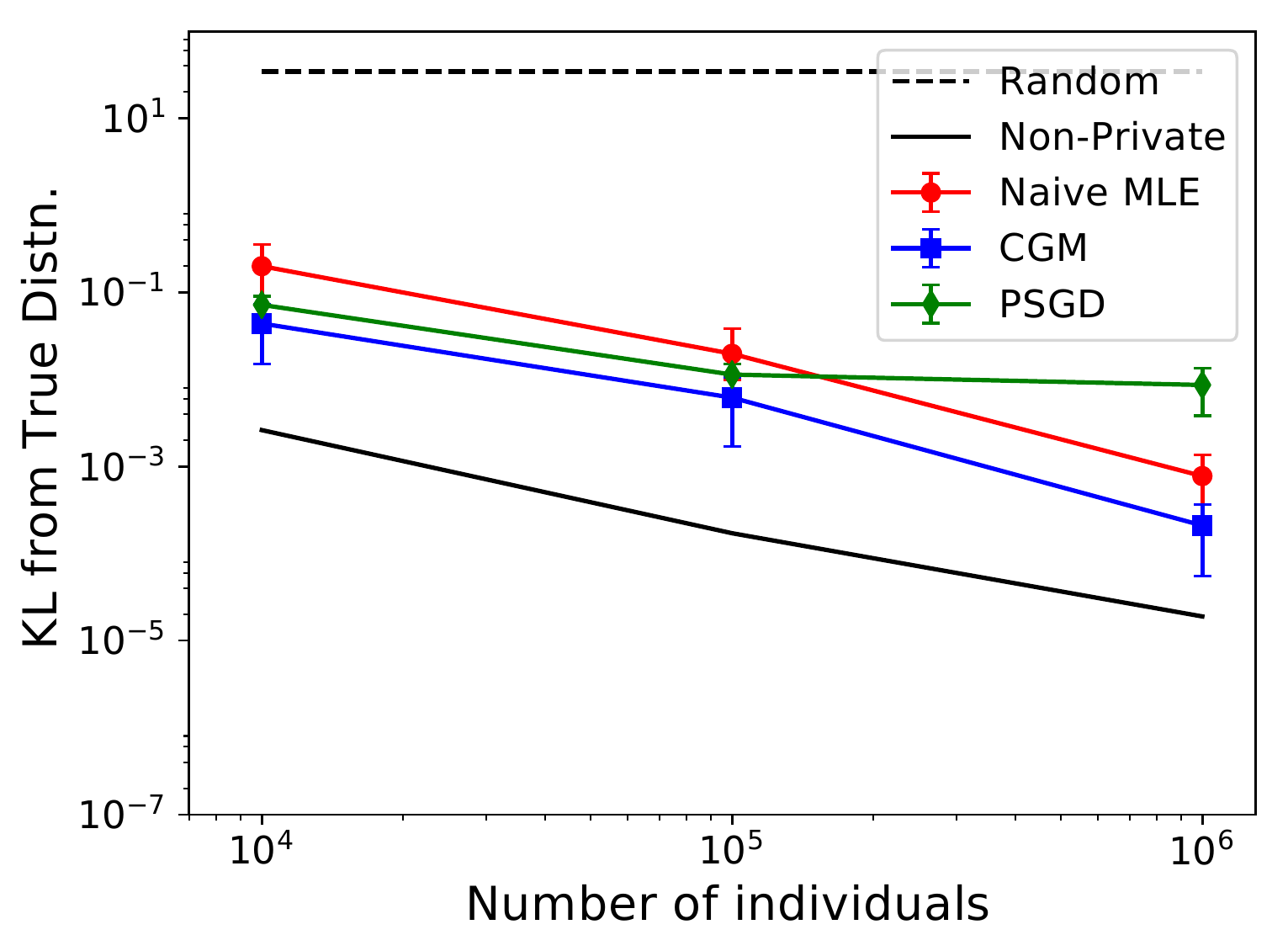}
            \label{subfig:thirdchain_1}
        }        
        \subfigure[$\epsilon = 1.0$]{
            \includegraphics[width=.23\textwidth]{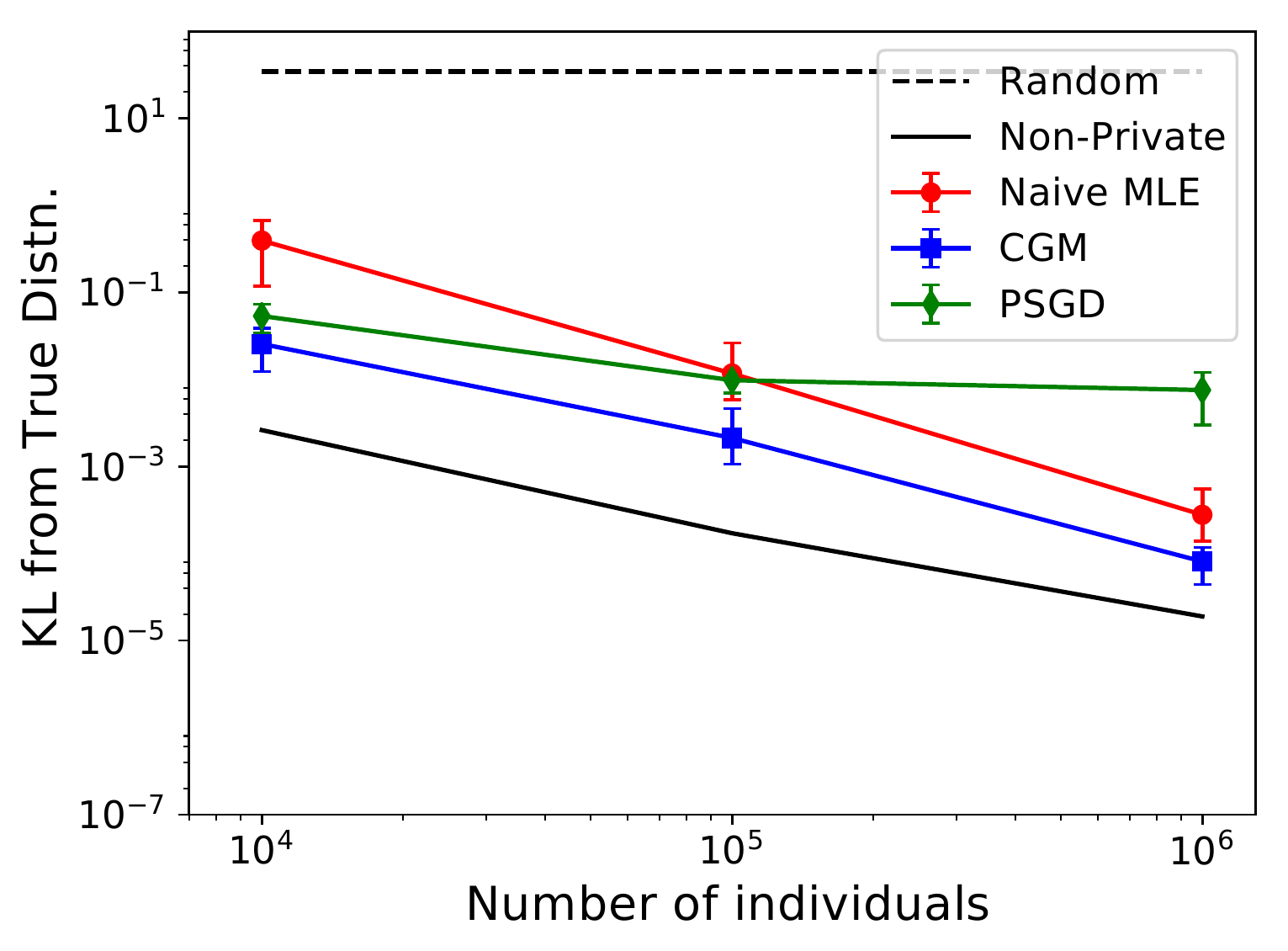}
            \label{subfig:thirdchain_0}
        }

        \subfigure[$\epsilon = 0.01$]{
            \includegraphics[width=.23\textwidth]{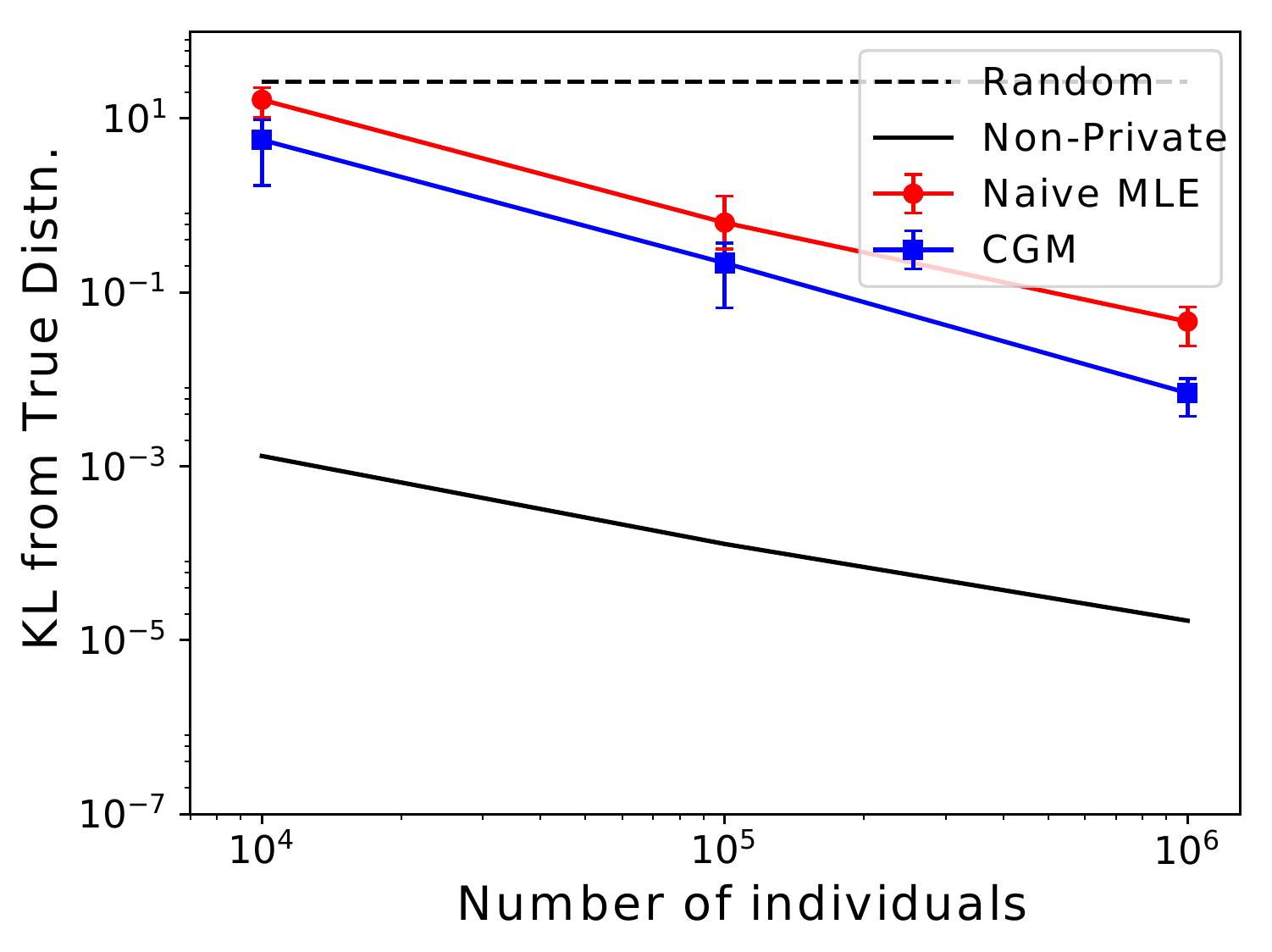}
            \label{subfig:erdosreyni_3}
        }
        \subfigure[$\epsilon = 0.1$]{
            \includegraphics[width=.23\textwidth]{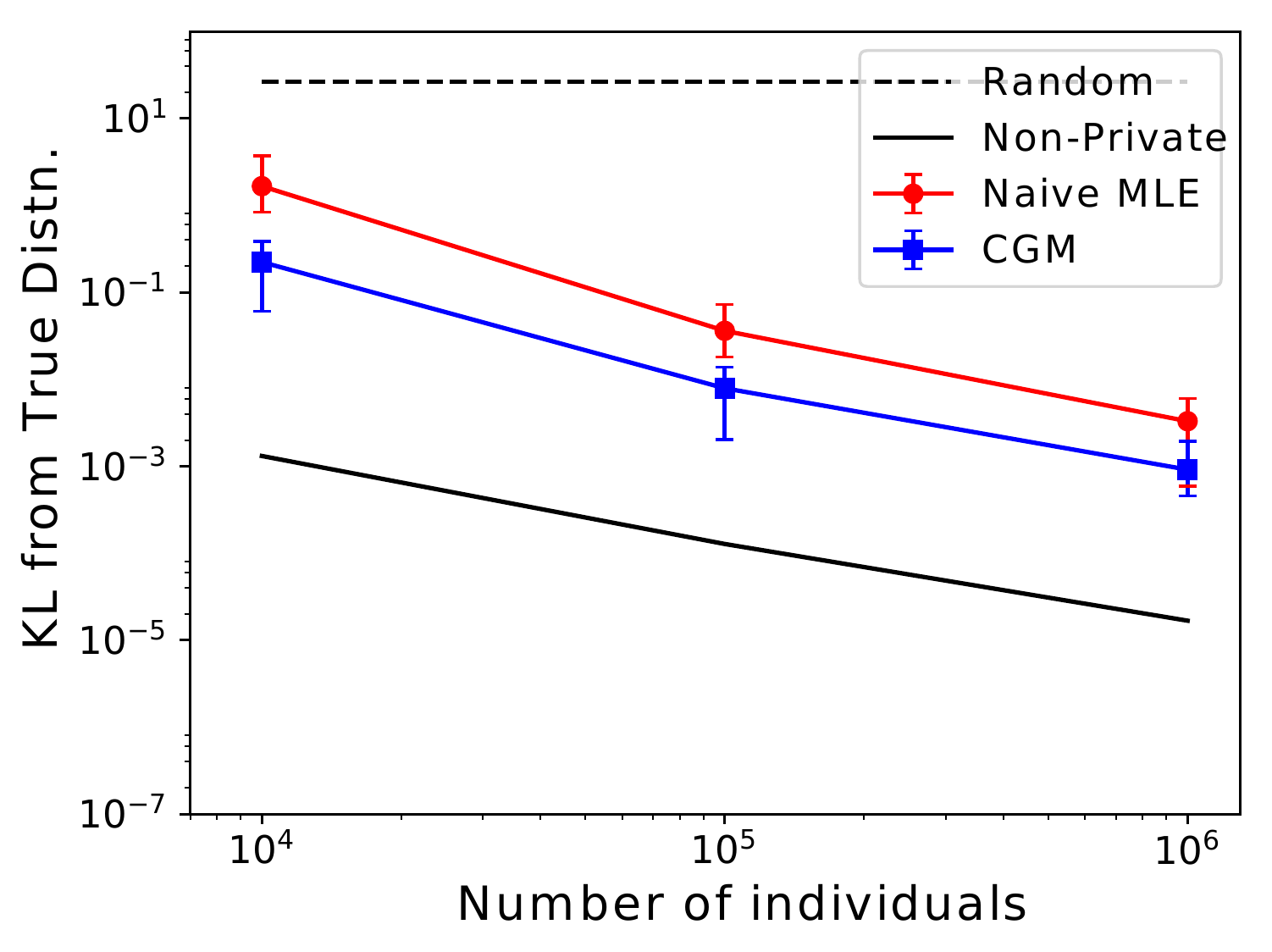}
            \label{subfig:erdosreyni_2}
        }
        \subfigure[$\epsilon = 0.5$]{
            \includegraphics[width=.23\textwidth]{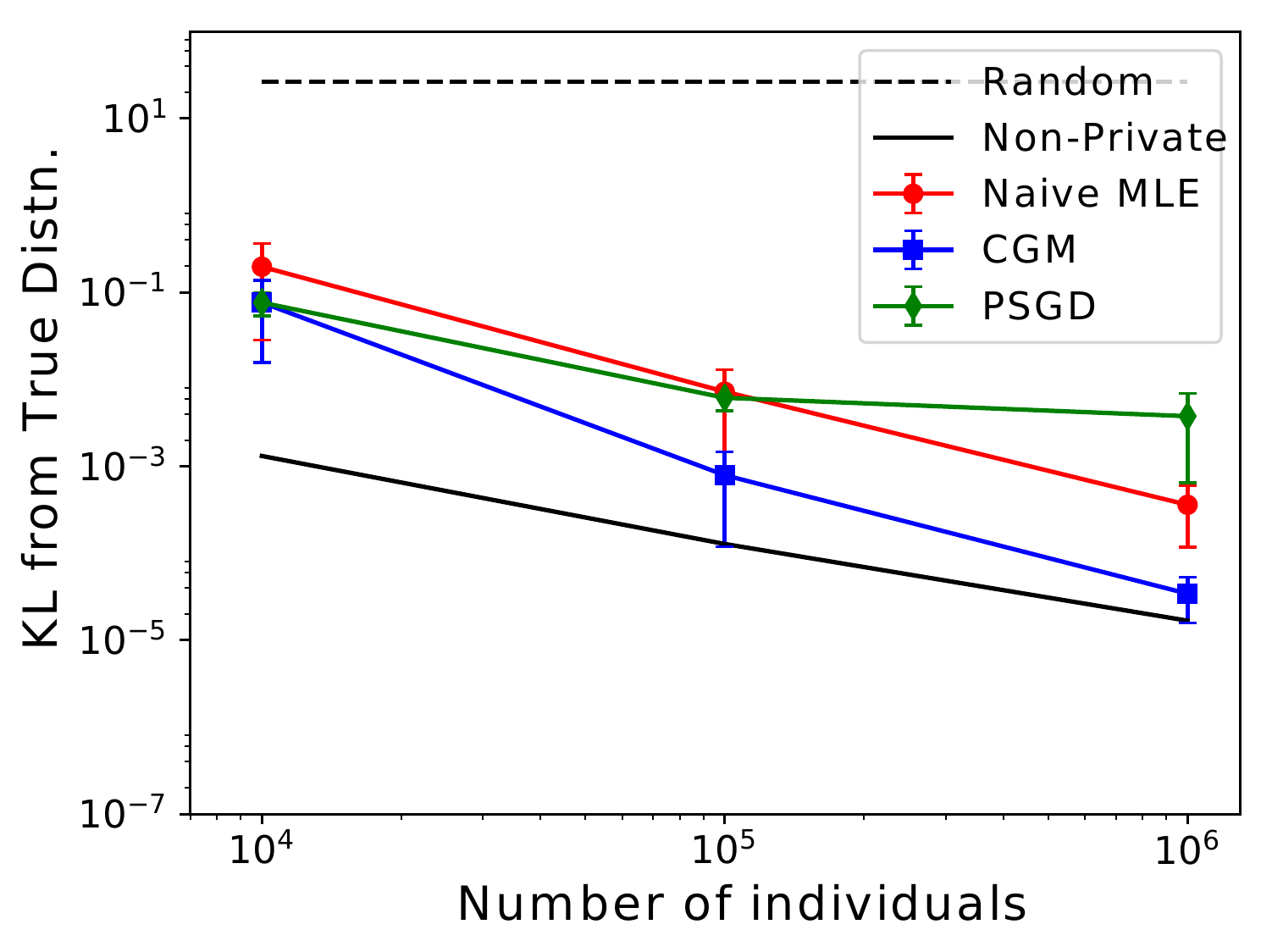}
            \label{subfig:erdosreyni_1}
        }        
        \subfigure[$\epsilon = 1.0$]{
            \includegraphics[width=.23\textwidth]{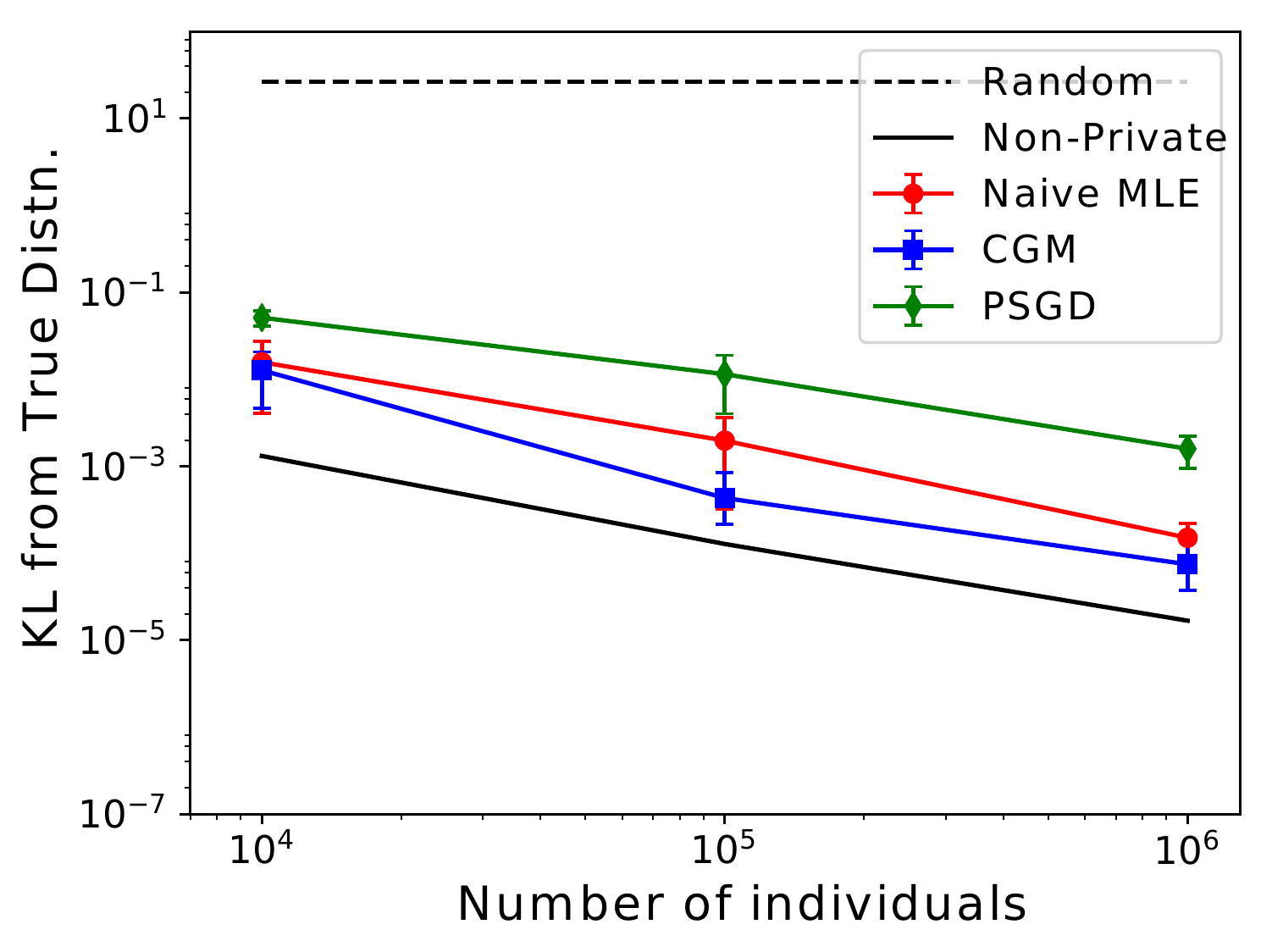}
            \label{subfig:erdosreyni_0}
        }
                
        \caption{Results on synthetic data generated from third-order chains, \subref{subfig:thirdchain_3}--\subref{subfig:thirdchain_0}, and connected Erd\H{o}s-R\'{e}yni random graphs \subref{subfig:erdosreyni_3}--\subref{subfig:erdosreyni_0}. Each column represents a different privacy level. Lower $\epsilon$ signifies stricter privacy guarantees. The x-axis measures population size. The y-axis is KL divergence from the true distribution. \label{fig:full_synthetic}}
    \end{figure*}

We evaluate two types of pairwise graphical models: third-order chains with edges between two nodes $i$ and $j$ if $1 \leq |i-j| \leq 3$, and (connected) Erd\H{o}s-R\'{e}yni (ER) random graphs. We report results for graphs of 10 nodes, where potentials on each edge are drawn from a Dirichlet distribution with concentration parameter of one; results are similar for smaller and larger models, models with different structures, and for different types of potentials. We vary data size $N$ and privacy parameter $\epsilon$. For each setting of model type, $N$, and $\epsilon$, we conduct 25 trials. The trials are nested, with five random populations and five replications per population, i.e.: $\n_i \sim p(\n), \y_{i,j} \sim p(\y \mid \n_i)$ for $i \in \{1,\ldots,5\}, j \in \{1, \ldots, 5\}$.
We measure the quality of learned models using KL divergence from the true distribution, and include for comparison two reference models: a random estimator and a non-private MLE estimator. The random estimator is obtained by randomly generating marginals $\bar{\mub}$ and then learning potentials via MLE.

\textbf{Results.} Figure~\ref{fig:full_synthetic} shows the results for the two models (top: third-order chain, bottom: ER) for different values of $N$ and $\epsilon$.
CGM improves upon Naive MLE for all models, privacy levels, and population sizes.
Recall that PSGD promises only $(\epsilon,\delta)$-differential privacy.  \savespace{Furthermore, because of the nature of this algorithm, there is an interplay between these two parameters and only some combinations are feasible.}  While $\delta$ is often assumed to be ``cryptographically small'', e.g., $O(2^{-N})$, we set $\delta$ to a relatively large value of $\delta = 1/N$.  Increasing $\delta$ weakens the privacy guarantee but enables PGSD to run on a wider range of $\epsilon$. However, even with this setting for $\delta$, some of the smaller values of $\epsilon$ are not attainable by PGSD and are omitted from those plots.

   \begin{figure}[t]
        \centering
        \subfigure[]{
            \includegraphics[width=.24\textwidth]{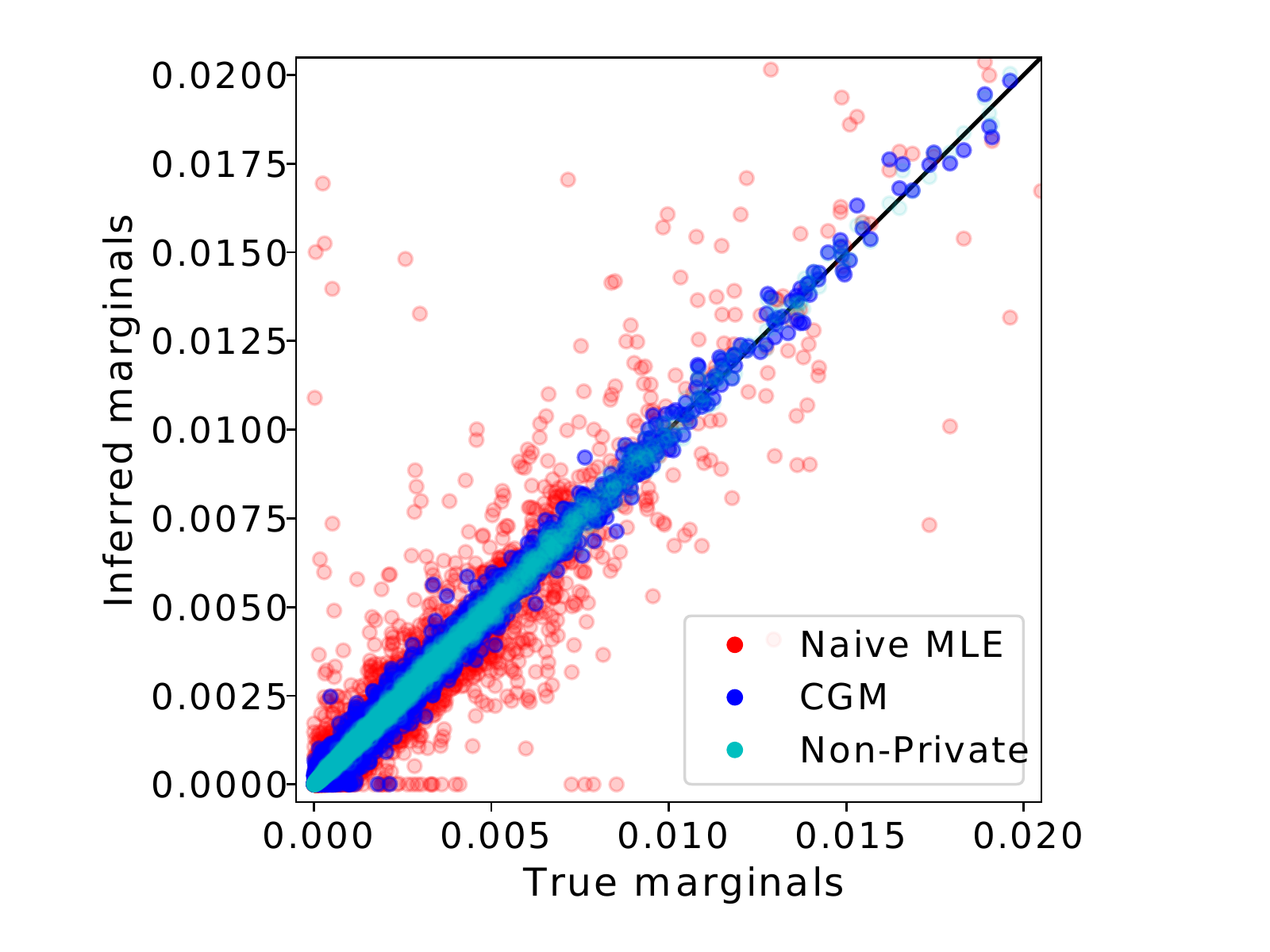}
            \label{fig:scatter}
        }
        \subfigure[]{
            \includegraphics[width=.21\textwidth]{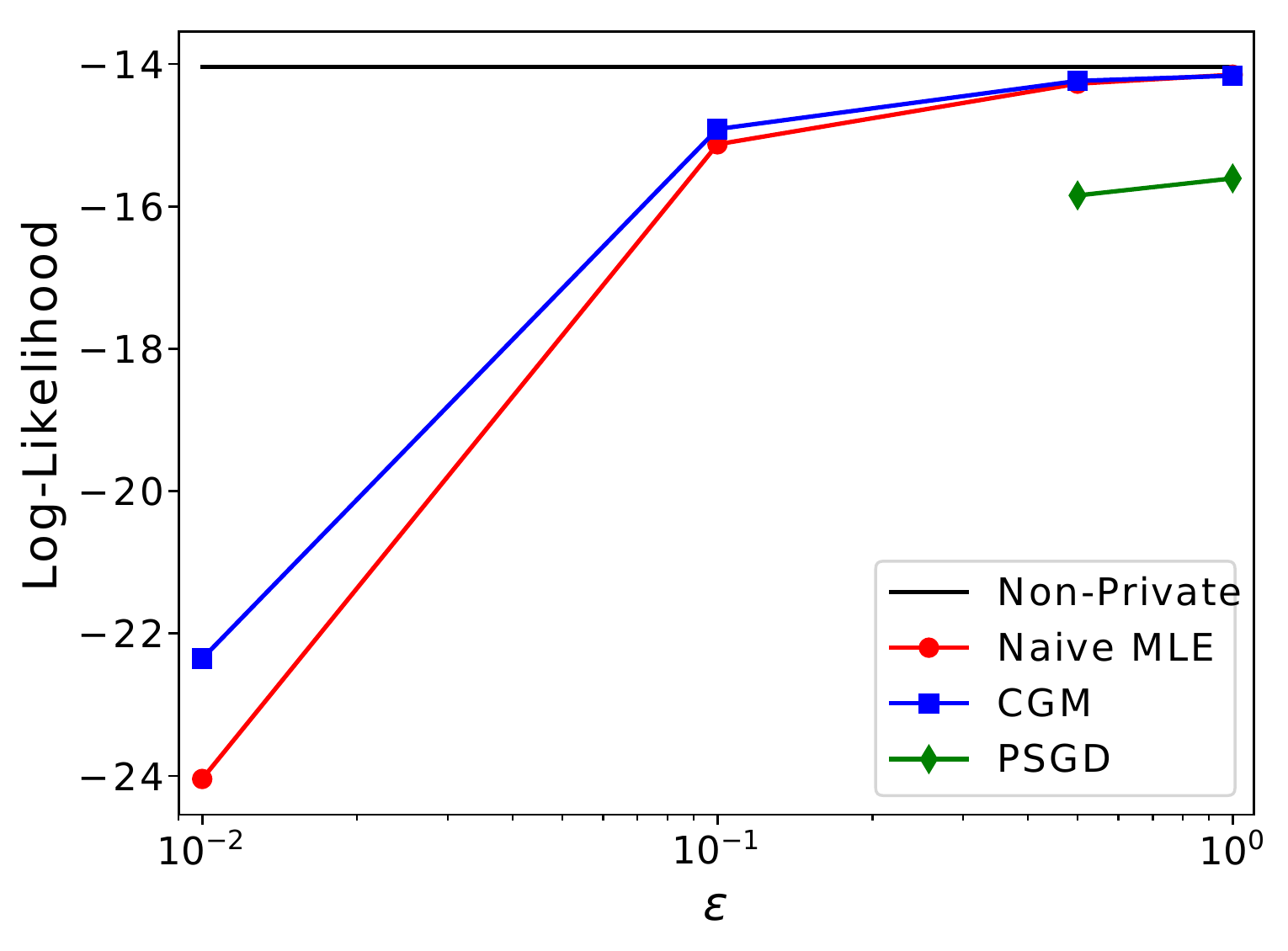}
            \label{fig:wifi}
        }
        
        \caption{\subref{fig:scatter} Scatter plots for true vs. inferred values of all edge marginals in an ER graph of 10 nodes with 20 states each. \subref{fig:wifi} Results for fitting a first-order chain on wifi data. The x-axis is privacy level; lower $\epsilon$ signifies stronger privacy guarantees. The y-axis is holdout log-likelihood.}
    \end{figure}

Figure~\ref{fig:scatter} shows a qualitative comparison of edge marginals of a single graph learned by the different methods, compared with the true model marginals; it is evident that CGM learns marginals that are much closer to both the true marginals and those learned by the non-private estimator than Naive MLE is able to learn. Naive MLE is the fastest method; CGM is approximately 4x/8x slower on third-order chains and ER graphs, respectively, and PSGD is approximately 27x/40x slower.




  \subsection{Wifi data} 
  \label{sub:wifi_data}

We study human mobility data in the form of connections to wifi access points throughout a highly-trafficked academic building over a twenty-one day period. We treat each (user ID, day) combination as an ``individual'', leading to 124,399 unique individuals; with this data preparation scheme, the unit of protection is one day's worth of a user's data. We discretize time by recording the location every 10 minutes, and assign null if the user is not connected \dan{to the network}. Our probability model $p(\x)$ is a pairwise graphical model over hour-long segments. Therefore, we break each individual's data into 24 one-hour long segments.

An individual now contributes 24 records to each contingency table for the model $p(\x)$. Therefore, the sensitivity is now $24$ times the number of edges (cliques). However, real data is typically sparse---i.e., an individual is typically observed only a small number of times over the observation period.  
Therefore, to reduce the sensitivity, the data is \emph{normalized} prior to calculating sufficient statistics, in a fashion similar to \cite{He:2015aa}. Each user contributes a value of $1/K$ to each contingency table, where $K$ is the number of edges $(x_s, x_t)$ for which the user's values are not both null. With this pre-processing in place, the sensitivity equals the number of edges in the model. A trade-off of this technique is that we bias the model towards individuals with fewer transitions, but we reduce the \dan{amount of noise} by limiting sensitivity caused by null--null transitions.

We reserve data from 25\% of the individuals for testing.
To compare different approaches, we apply Naive MLE, CGM, and PSGD to privately learn
parameters of a graphical model
from the training set (75\% of the data), with varying privacy levels. We then calculate holdout log-likelihood of the learned parameters on the test set. We again include a non-private method for reference, but in this case, all methods perform better than the random estimator, so we do not show it.

            

Figure \ref{fig:wifi} shows the results for fitting a time-homogeneous chain model (edges between adjacent time steps, every  potential $\psi(x_t, x_{t+1})$ is the same, and the model includes a node potential $\phi(x_1)$ so it can learn a time-stationary model). As in the synthetic data experiments, CGM improves upon naive MLE across all parameter regimes, and performance improves with population size $N$ and with weakening of privacy (larger $\epsilon$). Both methods outperform PSGD. Naive MLE is the fastest method; CGM is approximately 15x slower, and PSGD is approximately 46x slower.

                


\textbf{Acknowledgments}

This material is based upon work supported by the National Science Foundation under Grant Nos. 1409125, 1409143, 1421325, and 1617533.

\bibliography{cgm-privacy}
\bibliographystyle{icml2017}

\newpage 
\appendix

\onecolumn
\section{Extra Proofs}

\begin{proof}[Proof of Proposition~\ref{prop:sensitivity}]
It is well known that the local sensitivity of any contingency table with respect to our definition of $\nbrs(\mathbf{X})$ is one. This is easy to see from
the definition of $\n_C$ following Eq.~\eqref{eq:log-likelihood}: each
individual contributes a count of exactly one to each clique contingency table.
Since there are $|\C|$ tables, the local sensitivity is exactly $|\C|$ for all data sets, and, therefore, the sensitivity is the same.
\end{proof}

\begin{proof}[Proof of Proposition~\ref{prop:mse}]
Note that $n_C(i_C)$ is a sum of $N$ iid indicator variables, so $n_C(i_C) \sim
\text{Binomial}\big(N, \mu_C(i_C)\big)$, and $\Var\big(n_c(i_C)\big) = N \mu_C(i_C)\big(1-\mu_C(i_C)\big)$.
Now let $z \sim
\text{Laplace}(|\C|/\epsilon)$ and write:
\[
\bar{\mu}_C(i_C) = \frac{1}{N}\big(n_C(i_C) + z \big)
\]
Recall that $\E[z] = 0$ and $\Var(z) = 2|\C|^2/\epsilon^2$. We see immediately
that $\E[\bar{\mu}_C(i_C)] = \E\big[n_C(i_C)/N\big] = \mu_C(i_C)$.
Therefore, the estimator is unbiased and its mean-squared error is equal to its
variance. Since $n_C(i_C)$ and $z$ are independent, we have:
\begin{align*}
\Var\big(\bar{\mu}_C(i_C)\big) &= \frac{\Var\big(n_C(i_C)\big)}{N^2} + \frac{\Var(z)}{N^2} \\
&= \frac{\mu_C(i_C)\big(1-\mu_C(i_C)\big)}{N} + \frac{2|\C|^2}{N^2 \epsilon^2}
\end{align*}
The fact that $p(\x; \hat{\thetab})$ converges to $p(\x; \thetab)$ follows from
Proposition~\ref{prop:marginals} and the consistency of the marginals, as long
as the true marginals $\mub$ lie in the interior of the marginal polytope $\M$.
However, this is guaranteed because the true distribution $p(\x; \thetab)$ is
strictly positive.
\end{proof}

\begin{proof}[Proof of Proposition~\ref{prop:cgm-approx}]
After applying Stirling's approximation to $\log p(\n; \thetab)$ we obtain
\cite{nguyen2016approximate}:
\begin{equation}
\label{eq:H}
\log h(\n) \approx H(\n) = N\log N + \sum_{C \in \C} \hat{H}_C - \sum_{S \in \S}\nu(S) \hat{H}_S
\end{equation}
where we define $\hat{H}_A = -\sum_{i_A \in \X^{|A|}} n_A(i_A) \log n_A(i_A)$
for any $A \in \C \cup \S$. The term $\hat{H}_A$ is a scaled entropy. We can
rewrite it as:
\begin{align*}
\hat{H}_A
&= - N \sum_{i_A} \frac{n_A(i_A)}{N} \log \Big( \frac{n_A(i_A)}{N} \cdot N\Big) \\
&= - N \sum_{i_A} \hat{\mu}_A(i_A) \log \hat{\mu}_A(i_A) - N \sum_{i_A} \hat{\mu}_A(i_A) \log N \\
&= N H_A  -  N \log N
\end{align*}
where $H_A$ is now the entropy of the empirical marginal distribution
$\hat{\mub}_A = \n_A/N$. Since the total multiplicity of the separators is one
less than the number of cliques, when we substitute back into Eq.~\eqref{eq:H},
all of the $N \log N$ terms cancel, and we are left only with
\[
H(\n) = N \cdot \Big(\sum_{C \in \C(\T)} H_A - \sum_{S \in \S(\T)}\nu(S) H_A \Big)
\]
But, from standard arguments about the decomposition of entropy on junction
trees, the term in parentheses is exactly the entropy of distribution $q$
defined as:
\[
q(\x) =
\frac{\displaystyle \prod_{C \in \C} \prod_{i_C \in \X^{|C|}} \hat{\mu}_C(\x_C)}
{\displaystyle \prod_{S \in \S} \prod_{i_S \in \X^{|S|}} \hat{\mu}_S(\x_S)^{\nu(S)}},
\]
which factors according to $\C$ and can be written as $p(\x; \thetab)$ for
parameters $\thetab$ derived from the marginal probabilities. Although the
mapping from parameters to distributions is many-to-one, for any maginals
$\hat{\mub}$, there is a unique distribution $p(\x; \thetab)$ in the model family
that has marginals $\hat{\mub}$~\cite{wainwright2008graphical}, so this uniquely
defines $q(\x)$ as stated in the Proposition.
	
\end{proof}

\end{document}